\documentclass{article}

\usepackage{microtype}
\usepackage{graphicx}
\usepackage{subfigure}
\usepackage{booktabs} 
\usepackage{multirow}

\usepackage{hyperref}

\usepackage[accepted]{icml2025}

\usepackage{amsmath}
\usepackage{amssymb}
\usepackage{mathtools}
\usepackage{amsthm}
\usepackage{algorithm}
\usepackage{algorithmic}
\usepackage{tikz}
\usepackage{stfloats}

\newcommand{\dup}{\mathrm{d}}

\newcommand{\bR}{\mathbb{R}}

\newcommand{\mmm}{\hspace{10mm}}

\usepackage[capitalize,noabbrev]{cleveref}

\theoremstyle{plain}
\newtheorem{theorem}{Theorem}[section]
\newtheorem*{theorem*}{Theorem}

\newtheorem{lemma}[theorem]{Lemma}

\theoremstyle{definition}
\newtheorem{definition}[theorem]{Definition}

\newtheorem*{example*}{Example}
\theoremstyle{remark}

\usepackage[textsize=tiny]{todonotes}

\icmltitlerunning{Efficient, Accurate and Stable Gradients for Neural ODEs}

\begin{document}

\twocolumn[
\icmltitle{Efficient, Accurate and Stable Gradients for Neural ODEs}



\begin{icmlauthorlist}
\icmlauthor{Sam McCallum}{bath}
\icmlauthor{James Foster}{bath}
\end{icmlauthorlist}

\icmlaffiliation{bath}{Department of Mathematics, University of Bath, Bath, United Kingdom}

\icmlcorrespondingauthor{Sam McCallum}{sm2942@bath.ac.uk}

\icmlkeywords{Machine Learning, ICML, Neural ODEs, Numerical Methods, Backpropagation, Reversible Architectures}

\vskip 0.3in
]



\printAffiliationsAndNotice{}  

\begin{abstract}
Training Neural ODEs requires backpropagating through an ODE solve. The state-of-the-art backpropagation method is recursive checkpointing that balances recomputation with memory cost. Here, we introduce a class of algebraically reversible ODE solvers that significantly improve upon both the time and memory cost of recursive checkpointing. The reversible solvers presented calculate exact gradients, are high-order and numerically stable -- strictly improving on previous reversible architectures.
\end{abstract}

\section{Introduction}
\subsection{Neural ODEs}
Neural ODEs introduce a model class where the vector field of an ODE is parameterized as a neural network \cite{chen2018neural}. This strong prior on model space can offer an advantage in many problems.

For example, one of the most exciting areas for Neural ODEs is scientific modeling. This is emphasized by Hamiltonian and Lagrangian networks for learning physical systems \cite{greydanus2019hamiltonian, cranmer2020lagrangian}; and more generally by universal differential equations for combining theoretical models and neural network vector fields \cite{rackauckas2020universal}. Further, if theory is completely unknown, Neural ODEs can be used to identify latent dynamics from data. There are many examples of these approaches being applied to problems across the sciences \cite{lee2021parameterized, chen2022forecasting, lu2021neural, portwood2019turbulence, boral2024neural}.

The continuous latent state encoded by Neural ODEs can also be beneficial in time-series modeling, where data is often partially observed and irregularly sampled. The application to time-series was originally explored by latent ODEs \cite{rubanova2019latent} and built upon by Neural Controlled Differential Equations (CDEs) \cite{kidger2020neural} and Stochastic Differential Equations (SDEs) \cite{li2020scalable, kidger2021neural, issa2024non}.

\begin{figure}
    \centering
    \includegraphics[width=\linewidth]{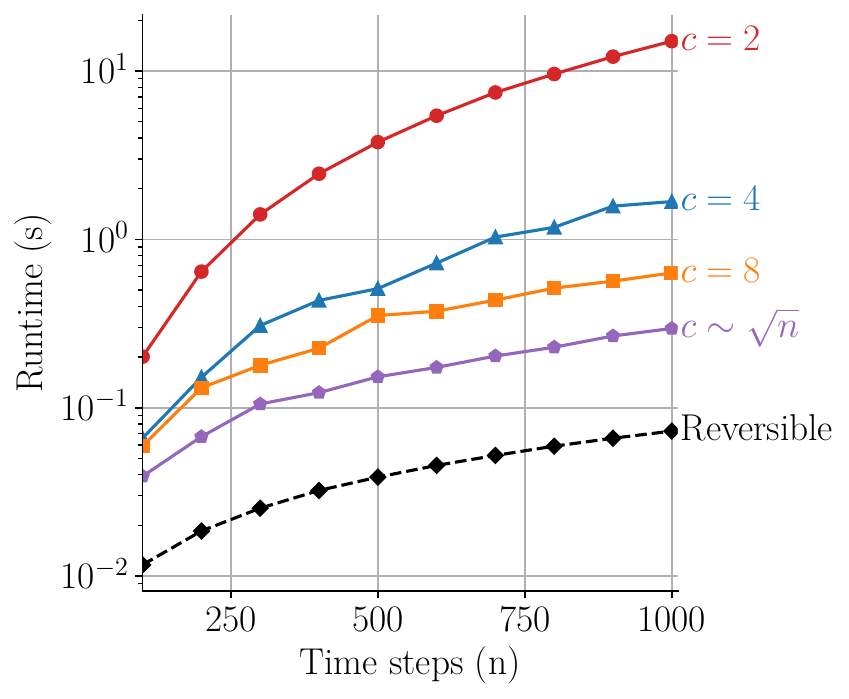}
    \caption{Runtime complexity of reversible backpropagation algorithm vs recursive checkpointing with $c$ checkpoints.}
    \label{fig:complexity}
\end{figure}
\subsection{Training}
Training Neural ODEs requires that we backpropagate through the ODE solve and proceeds by two main approaches: standard automatic differentiation through the internal numerical solver operations (discretize-then-optimize) or via the continuous adjoint method (optimize-then-discretize) \cite{kidger2022neural}.

Discretize-then-optimize is generally the preferred training method as gradient calculation is exact, rather than the approximation made by optimize-then-discretize. However, the memory cost of storing all numerical operations can become prohibitively large. 

To keep memory cost low, checkpointing algorithms are used that balance recomputation and memory usage by only storing a subset of the numerical operations \cite{stumm2010new, griewank1992achieving, gholami2019anode}. 

\newpage

Specifically, online recursive checkpointing algorithms guarantee an $O(n\log n)$ time cost while storing $O(\sqrt{n})$ checkpoints, where $n$ is the computation length. An online algorithm is required to handle the case where $n$ is unknown in advance, for example when using adaptive step size methods.

The online recursive checkpointing algorithm from \cite{stumm2010new} will form the baseline method for our comparisons here.

\subsection{Algebraically reversible Solvers}
Interestingly, it is possible to backpropagate through ODEs without storing any intermediate numerical operations $\emph{and}$ calculate exact gradients. This is realized by constructing an algebraically reversible solver, whereby the solver state at step $n$ can be reconstructed exactly in closed form from the solver state at step $n+1$. For example, symplectic solvers for Hamiltonian systems are intrinsically reversible \cite{greydanus2019hamiltonian}.

In general, there have been two previously proposed algebraically reversible solvers: Zhuang et al. \cite{zhuang2021mali} use the Asynchronous Leapfrog (ALF) method that is a second-order reversible ODE solver; and, Kidger et al. \cite{kidger2021efficient} devise a reversible Heun method that is a second-order ODE solver and $0.5$ strong order SDE solver.

Reversible solvers realize an $O(n)$ time and $O(1)$ memory backpropagation algorithm, improving upon the complexity of recursive checkpointing. However, current reversible solvers are low-order and suffer from poor numerical stability. The use of these methods in practice is therefore limited.

\begin{figure*}[t]
\vskip 0.2in
    \centering
    \usetikzlibrary{positioning, shapes, arrows}

\begin{tikzpicture}[
    node distance=1.5cm and 2.5cm,
    sum/.style={circle, draw, minimum size=5mm, inner sep=0pt},
    func/.style={circle, draw, thick, minimum size=6mm, text width=3.5mm, align=center},
    arrow/.style={->, >=stealth, thick}
]

\node at (-1, 2) {\textbf{(a)}};

\node (y0) at (0,2) {$y_n$};
\node (z0) at (0,0) {$z_n$};
\node[func, fill=red!20, draw=red!80] (psi_forward) at (3.0,1) {$\Psi_h$};
\node[sum] (plus1) at (2.3, 2) {$+$};
\node[sum] (plus2) at (3.5,2) {$+$};
\node[circle, fill=black, inner sep=1.5pt] (startz1) at (1.3, 0) {};
\node[circle, fill=black, inner sep=1.5pt] (startz2) at (2.5, 0) {};
\node[circle, fill=black, inner sep=1.5pt] (starty1) at (4.0, 2) {};
\node[func, fill=blue!20, draw=blue!80] (psi_backward) at (4.5,1) {$\Psi_{\scalebox{0.75}[1.0]{-}h}$};
\node[sum] (minus1) at (5,0) {$-$};
\node (z1) at (6,0) {$z_{n+1}$};
\node (y1) at (6,2) {$y_{n+1}$};

\draw[arrow] (startz2) -- (psi_forward);
\draw[arrow] (y0) -- (plus1) node[pos=0.5, above] {$\times \lambda$};
\draw[arrow] (plus1) -- (plus2);
\draw[arrow] (startz1) -- (plus1) node[pos=0.55, left] {$\times (1-\lambda)$};
\draw[arrow] (psi_forward) -- (plus2);
\draw[arrow] (plus2) -- (y1);
\draw[arrow] (starty1) -- (psi_backward);
\draw[arrow] (psi_backward) -- (minus1);
\draw[arrow] (z0) -- (minus1);
\draw[arrow] (minus1) -- (z1);

\node at (7, 2) {\textbf{(b)}};

\node (y0) at (8,2) {$y_n$};
\node (z0) at (8,0) {$z_n$};
\node[func, fill=red!20, draw=red!80] (psi_forward) at (11,1) {$\Psi_h$};
\node[sum] (plus1) at (9.3, 2) {$-$};
\node[sum] (plus2) at (10.5,2) {$-$};
\node[circle, fill=black, inner sep=1.5pt] (startz1) at (10.3, 0) {};
\node[circle, fill=black, inner sep=1.5pt] (startz2) at (11.5, 0) {};
\node[circle, fill=black, inner sep=1.5pt] (starty1) at (13, 2) {};
\node[func, fill=blue!20, draw=blue!80] (psi_backward) at (12.5,1) {$\Psi_{\scalebox{0.75}[1.0]{-}h}$};
\node[sum] (minus1) at (12,0) {$+$};
\node (z1) at (14,0) {$z_{n+1}$};
\node (y1) at (14,2) {$y_{n+1}$};

\draw[arrow] (startz2) -- (psi_forward);
\draw[arrow] (plus1) -- (y0) node[pos=0.5, above] {$\times \lambda^{-1}$};
\draw[arrow] (plus2) -- (plus1);
\draw[arrow] (startz1) -- (plus1) node[pos=0.5, left] {$\times (1-\lambda)$};
\draw[arrow] (psi_forward) -- (plus2);
\draw[arrow] (y1) -- (plus2);
\draw[arrow] (starty1) -- (psi_backward);
\draw[arrow] (psi_backward) -- (minus1);
\draw[arrow] (minus1) -- (z0);
\draw[arrow] (z1) -- (minus1);

\end{tikzpicture}
    \caption{Computation graph of the reversible method. \textbf{(a)} Forward solve. \textbf{(b)} Backward solve.}
    \label{fig:comp-graph}
\vskip 0.2in
\end{figure*}

\subsection{Contributions}
In this work we address both the low-order convergence and stability issues of previous reversible solvers for Neural ODEs.

We present a general class of algebraically reversible solvers that allows any single-step numerical solver to be made reversible. This class of reversible solvers calculate exact gradients and display the following properties:
\begin{enumerate}
    \item $O(n)$ time, $O(1)$ memory complexity (Algorithm \ref{alg:backprop}),
    \item high-order convergence (Theorem \ref{thm:convergence}),
    \item improved numerical stability (Theorem \ref{thm:stability}).
\end{enumerate}

\section{Reversible Solvers}
We introduce a class of reversible solvers where any single-step numerical solver can be made reversible. The algebraic reversibility property allows one to dynamically recompute the forward solve in closed form during backpropagation, thereby obtaining exact gradients in $O(n)$ time and $O(1)$ memory.

We will write our Neural ODE system as
\begin{equation}
    \label{eq:neural-ode}
    \frac{\dup y}{\dup t}(t) = f_\theta(t, y(t)), \hspace{8mm} y(0) = y_0,
\end{equation}
where $y(t)\in \mathbb{R}^d$ is the solution to \eqref{eq:neural-ode} and $f_\theta: \mathbb{R} \times \mathbb{R}^d \rightarrow \mathbb{R}^d$ is a neural network parameterized by $\theta \in \mathbb{R}^m$. The initial condition is given by $y_0\in \mathbb{R}^d$.

\subsection{Forward solve}
Suppose we have a single-step numerical ODE solver, $\Psi_h(t, y) : \mathbb{R} \times \mathbb{R}^d \rightarrow \mathbb{R}^d$, such that a numerical step is given by $y_{n+1}=y_n+\Psi_h(t, y)$, where $h$ is the step size. Then we construct a reversible numerical solution $\{y_n, z_n\}_{n\geq 0}$ by
\begin{equation}
    \label{eq:reversible-solver-forward}
    \begin{aligned}
        y_{n+1} &= \lambda y_n+(1-\lambda)z_n + \Psi_h(t_n, z_n), \\
        z_{n+1} &= z_n - \Psi_{-h}(t_{n+1}, y_{n+1}),
    \end{aligned}
\end{equation}
where $\lambda\in(0, 1]$ is a coupling parameter. The initial state is given by $y_0=z_0=y(0)$ and $t_n=nh$ for each $n\geq 0$.

\subsection{Backward solve}
The numerical scheme in \eqref{eq:reversible-solver-forward} is algebraically reversible, with the reverse scheme given by
\begin{equation}
    \label{eq:reversible-solver-backward}
    \begin{aligned}
        z_n &= z_{n+1}+\Psi_{-h}(t_{n+1}, y_{n+1}), \\
        y_n &= \lambda^{-1}y_{n+1} + \left(1-\lambda^{-1}\right)z_n - \lambda^{-1}\Psi_h(t_n, z_n).
    \end{aligned}
\end{equation}
This construction allows us to step between $(y_n, z_n) \leftrightarrow (y_{n+1}, z_{n+1})$ in closed form. It is this property that results in exact gradient calculation as we can dynamically recompute the forward computation graph exactly.

The above construction of reversibility shares familiarities with other reversible architectures in machine learning. For example, we see a similar coupling between the evolving state in Reversible Networks (RevNets) \cite{gomez2017reversible} and in coupling layers for normalizing flows \cite{dinh2015nice}. However, the coupling parameter $\lambda$ is absent from these architectures; we show later that $\lambda$ is essential for numerical stability when solving ODEs (see Theorem \ref{thm:stability}).

\subsection{Convergence}
The convergence order of the reversible solver is inherited from the base solver $\Psi$. That is, provided $\Psi$ is a method with $k$-th order convergence then the reversible scheme will also have $k$-th order convergence. This result is illustrated by Theorem \ref{thm:convergence}.

For example, we could choose $\Psi$ to be a step of a fourth order Runge-Kutta method (RK4). In which case, the reversible solver in \eqref{eq:reversible-solver-forward} would correspond to a reversible RK4 method and also display fourth order convergence.

We therefore see that a reversible solver of arbitrarily high order may be constructed from this method. This is in contrast to previous reversible methods where the convergence order is baked in.

\vskip 0.1in
\begin{theorem}
\label{thm:convergence}
For a fixed time-horizon $T>0$, we consider the Neural ODE in \eqref{eq:neural-ode} over $[0, T]$. Let $T=Nh$ where $N>0$ denotes the number of steps and $h>0$ is the step size. 

Let $\Psi$ be a $k$-th order ODE solver satisfying the Lipschitz condition (see \ref{def:lip-solver}) and consider the reversible solution $\{y_n, z_n\}_{n\geq 0}$, given by \eqref{eq:reversible-solver-forward}. Then there exist constants $h_\text{max}, C>0$ such that, for $h\in[0, h_\text{max}]$,
\begin{equation}
    ||y_n - y(t_n)|| \leq Ch^k.
\end{equation}
\end{theorem}

\begin{proof}
    The proof of Theorem \ref{thm:convergence} is given in Appendix \ref{app:convergence}.
\end{proof}

\begin{figure}
    \centering
    \includegraphics[width=\linewidth]{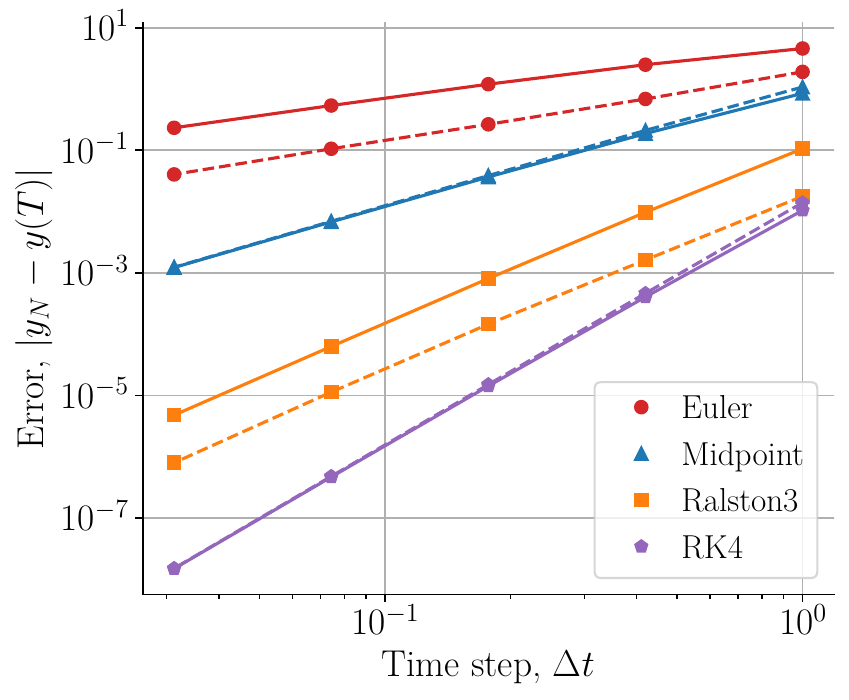}
    \caption{Reversible solver convergence (dashed) is inherited from base solver (solid).}
    \label{fig:convergence}
\end{figure}

\subsection{Backpropagation Algorithm}
We now discuss the key result of reversible solvers: exact gradient backpropagation in $O(n)$ time and $O(1)$ memory. The complexity results from dynamically recomputing the forward solve during backpropagation, requiring only the terminal solver state be stored in memory.

The backpropagation algorithm through one solver step is presented in Algorithm \ref{alg:backprop} and follows the construction of reverse-mode automatic differentiation \cite{baydin2018automatic}.

\vskip 0.1in

\begin{algorithm}
\caption{Reversible backpropagation}\label{alg:backprop}
    \begin{algorithmic}
        \STATE {\bfseries Input:} $t_{n+1}, y_{n+1}, z_{n+1}, \bar{y}_{n+1}, \bar{z}_{n+1}, \bar{\theta}$ 
        \vspace{1ex}
        \STATE \# Backward step
        \STATE $z_n = z_{n+1}+\Psi_h(t_{n+1}, y_{n+1})$
        \vspace{1ex}
        \STATE $y_n = \lambda^{-1}y_{n+1} + \left(1-\lambda^{-1}\right)z_n - \lambda^{-1}\Psi_h(t_n, z_n)$
        \vspace{1ex}
        \STATE $t_n = t_{n+1} - h$
        \vspace{2ex}
        \STATE \# Gradients
        \STATE $\bar{y}_{n+1} \leftarrow \bar{y}_{n+1} - \bar{z}_{n+1}\displaystyle\frac{\partial \Psi_{-h}(t_{n+1}, y_{n+1})}{\partial y_{n+1}}$
        \STATE $\bar{y}_n = \lambda\bar{y}_{n+1}$
        \STATE $\bar{z}_n = \bar{z}_{n+1} + (1-\lambda)\bar{y}_{n+1} + \bar{y}_{n+1}\displaystyle\frac{\partial \Psi_h(t_n, z_n)}{\partial z_n}$
        \STATE $\bar{\theta} \leftarrow \bar{\theta}-\bar{z}_{n+1}\displaystyle\frac{\partial \Psi_{-h}(t_{n+1}, y_{n+1})}{\partial \theta}  + \bar{y}_{n+1}\displaystyle\frac{\Psi_h(t_n, z_n)}{\partial \theta}$
        \STATE {\bfseries Return:} $t_n, y_n, z_n, \bar{y}_n, \bar{z}_n, \bar{\theta}$
    \end{algorithmic}
\end{algorithm}

\vskip 0.1in

We define a scalar-valued loss function $L(y_N) : \mathbb{R}^d \rightarrow \mathbb{R}$ on the terminal value of the numerical solve $y_N \approx y(T)$. It is possible to define a loss function on any subset of the numerical solution $\{y_n\}_{0\leq n \leq N}$ but we choose to depend on $y_N$ for ease of presentation. 

The gradients (adjoints) are defined as $\bar{v}=\partial L(y_N)/\partial v$. Note that the presence of an adjoint to the left of a Jacobian, $\bar{v} \frac{\partial f}{\partial x}$, indicates a vector-Jacobian product. This can be efficiently computed using automatic differentiation.

The algorithm proceeds by first taking a backward step from $n+1$ to $n$, recovering the state at $n$ exactly. Then we backpropagate the adjoints from $n+1$ to $n$ by tracing the computation graph of the forward solve, applying the chain rule as we go. We write the step size as $h$ which can be a constant or selected by an adaptive time stepping algorithm.

The vector-Jacobian products in Algorithm \ref{alg:backprop} are computable for an arbitrary (differentiable) step function $\Psi$. This product requires no knowledge on the structure of $\Psi$ a-priori and no local forward evaluations as $\Psi$ is evaluated on the backward solve.

\subsection{Stability}
One key ingredient missing from previous reversible solvers is numerical stability. Here we show that the reversible scheme introduced above in \eqref{eq:reversible-solver-forward} has a non-zero linear stability region for $\Psi$ given by any explicit Runge-Kutta method. This is shown by Theorem \ref{thm:stability}.

We prove stability, as is customary, for the real-valued linear test problem \cite{stewart2022numerical}. 

\vskip 0.1in

\begin{definition}[Linear Stability]
    \label{def:linear-stability}
    Let $y\in \mathbb{R}^d$ and consider the ODE,
    \begin{equation}
        \label{eq:linear-stability}
        \frac{\dup y}{\dup t} = \alpha y, \hspace{8mm} y(0) = y_0,
    \end{equation}
    where $\alpha < 0$ is a negative real constant and $y_0 \in \mathbb{R}^d$ is a non-zero initial condition. A numerical solution $\{y_n\}_{n\geq 0}$ to \eqref{eq:linear-stability} is linearly stable if the spectral radius (maximum eigenvalue) $\rho(T)<1$, where $T\in \mathbb{R}^{d \times d}$ is given by
    \begin{equation}
        y_{n+1} = Ty_n.
    \end{equation}
\end{definition}

\vskip 0.1in

This definition expresses the condition that a numerical method applied to \eqref{eq:linear-stability} should decay to zero as $t\rightarrow\infty$. For this to hold we require that the maximum eigenvalue of $T<1$. Note that for the reversible scheme in \eqref{eq:reversible-solver-forward} we require that the pair $\{y_n, z_n\}$ tends to zero.

The theorem is simplified by the concept of transfer functions for Runge-Kutta methods, such that we can write $\Psi_h(t_n, y_n)=R(h\alpha)y_n$ for a `transfer function' $R(h\alpha)$ with step size $h$. This is a commonly used result in stability analysis, see \cite{stewart2022numerical} for a complete discussion.

\newpage

\begin{theorem}
    \label{thm:stability}
    Let $\Psi$ be given by an explicit Runge-Kutta solver. Then the reversible numerical solution $\{y_n, z_n\}_{n\geq 0}$ given by \eqref{eq:reversible-solver-forward} is linearly stable iff
    \begin{equation}
        \label{eq:stability-condition}
        |\Gamma| < 1+\lambda,
    \end{equation}
    where
    \begin{equation}
        \Gamma = 1+\lambda - (1-\lambda)R(-h\alpha)-R(-h\alpha)R(h\alpha).
    \end{equation}
\end{theorem}
\begin{proof}
    The proof of Theorem \ref{thm:stability} is given in Appendix \ref{app:stability}.
\end{proof}

\vskip 0.1in

\begin{example*} Theorem \ref{thm:stability} is best illustrated by an example. Consider using Euler's method for the base solver $\Psi$, then $R(h\alpha)=h\alpha$. Substituting into the stability condition \eqref{eq:stability-condition}, we get $h\alpha > \lambda - 1$. As $\lambda\rightarrow 1$, the stability region tends to zero. And, as $\lambda\rightarrow 0$ we recover half the stability region of the base Euler solver.
\end{example*}

This example illustrates some key points. The stability region of reversible $\Psi$ is decreased in comparison to the stability region of $\Psi$, where $\Psi$ is the base solver. However, the stability region of the reversible solver is non-zero and for any coupling parameter $\lambda\in(0, 1)$ there exists a step size $h$ such that the reversible scheme is linearly stable. 

This is in contrast to previous reversible architectures, ALF \cite{zhuang2021mali} and reversible Heun \cite{kidger2021efficient}, that are nowhere linearly stable for any step size $h$.

We see from Theorem \ref{thm:stability} that the choice of coupling parameter $\lambda$ is key to the reversible method -- stability is increased on the forward solve by decreasing $\lambda$. However, decreasing $\lambda$ will reduce the numerical stability of the backward solve due to the presence of $\lambda^{-1}$ terms. To balance the stability it is therefore important that $\lambda$ remains close to $1$ and in practice we find $\lambda\in[0.99, 0.999]$ to be good choices.

\subsection{Adaptive step sizes}
It is straightforward to extend the reversible solvers to use adaptive step sizes. This follows naturally from the construction of the forward step \eqref{eq:reversible-solver-forward} due to the dependence on the base solver step $\Psi$. That is, if the base solver $\Psi$ is capable of providing local error estimates then the reversible method can inherit the estimate. 

This allows us to obtain an adaptive reversible solver from an adaptive base solver. We will use an adaptive reversible solver in experiment \ref{subsec:pendulum} to handle chaotic system dynamics.

We choose to inherit the error estimate provided by the local forward step, $\Psi_h$. It is also possible to include the error estimate provided by the local backward step, $\Psi_{-h}$. In general, two error estimates are available per step at $t_n$ and $t_{n+1}$ which may prove useful for stiff problems.

\section{Experiments}
We perform three experiments that compare the accuracy, runtime and memory cost of reversible solvers to recursive checkpointing. The experiments focus on scientific modeling and discovery of latent dynamics from data, an application where Neural ODEs appear likely to have the highest impact.

First, we consider discovery of Chandrasekhar's White Dwarf equation from generated data. Second, we investigate a real-data coupled oscillator system \cite{schmidt2009distilling}. Finally, we extend the reversible solvers to manage chaotic system dynamics using adaptive step sizes on a real-data double pendulum experiment.

\begin{table*}[t]
\centering
\caption{Memory usage and runtime for Chandrasekhar experiment. Mean $\pm$ standard deviation over three repeats.}
\vskip 0.1in
\begin{tabular}{l|c|c c c c}
Method & \shortstack{Memory\\(checkpoints)} & \multicolumn{4}{c}{\shortstack{Runtime\\(min)}}  \\ \hline
 & & Euler & Midpoint & Ralston3 & RK4 \\
\hline
\textbf{Reversible} & $\mathbf{2}$ & $\mathbf{1.5 \pm 0.3}$ & $\mathbf{1.7 \pm 0.4}$ & $\mathbf{1.8 \pm 0.2}$ & $\mathbf{2.2 \pm 0.3}$ \\ 
\hline
\multirow[c]{5}{*}{Recursive} & $2$ & $264.0 \pm 17.5$ &  $280.0 \pm 13.7$ &  $259.9 \pm 22.9$ &  $254.3 \pm 16.1$ \\ 
 & $4$ & $30.5 \pm 2.3$ &  $30.3 \pm 1.6$ &  $29.8 \pm 2.2$ &  $29.2 \pm 1.8$ \\ 
 & $8$ & $10.7 \pm 1.4$ &  $10.6 \pm 1.1$ &  $10.0 \pm 0.4$ &  $9.9 \pm 0.3$ \\ 
 & $16$ & $8.9 \pm 1.2$ &  $9.6 \pm 0.6$ &  $8.4 \pm 0.2$ &  $8.2 \pm 0.2$ \\ 
 & $32$ & $8.1 \pm 0.8$ &  $8.7 \pm 0.7$ &  $7.6 \pm 0.2$ &  $7.5 \pm 0.2$ \\ 
 & $44$ & $5.2 \pm 0.7$ &  $5.5 \pm 0.8$ &  $4.90 \pm 0.04$ &  $4.95 \pm 0.02$ \\ 
\end{tabular}

\label{tab:chandrasekhar}
\end{table*}

\subsection{Chandrasekhar's White Dwarf Equation}
The first task we consider is discovery of Chandrasekhar's white dwarf equation from generated data. The system describes the density of a white dwarf $\varphi$ as a function of the radial distance $r$ from the center of the star. The density follows the second-order ODE,
\begin{align*}
    \frac{1}{r^2}\frac{\dup}{\dup r}&\left(r^2\frac{\dup \varphi}{\dup r}\right) + \left(\varphi^2 - C\right)^{3/2}=0, \\
    &\varphi(0)=1, \hspace{2mm} \frac{\dup\varphi}{\dup r}(0)=0,
\end{align*}
where $\sqrt{C}$ is a constant that defines the lower bound on the density $\varphi$. The system provides a strong test case for Neural ODEs as the vector field is both non-linear and `time-dependent'.

We generate training data by simulating the white dwarf system over $r\in [0, 5]$ with $C=0.001$. The task is to learn the dynamics of the white dwarf system. 

We parameterize the Neural ODE in \eqref{eq:neural-ode} with a feed-forward neural network vector field. The model is trained to minimise the mean-squared-error. Gradient calculation is performed by both the reversible method and recursive checkpointing for a range of numerical solvers: Euler, Midpoint, Ralston’s 3rd order method (Ralston3) and the classic Runge-Kutta 4 method (RK4).

The runtime and memory usage of both training methods is shown in Table \ref{tab:chandrasekhar}. We compare to recursive checkpointing over a range of checkpoints $c$ up to the optimal memory/runtime trade-off $c\sim\sqrt{n}$. Note that the memory usage of the reversible solver is equivalent to 2 checkpoints as the method must store the augmented state $\{y_n, z_n\}$.

We see from Table \ref{tab:chandrasekhar} that reversible backpropagation is at least $2.5\times$ faster than recursive checkpointing while using $22\times$ less memory. Further, reversible backpropagation is at least $100\times$ faster for the same memory usage. These results hold across all solvers considered. The mean final loss obtained is the same for both methods, $0.9\times 10^{-4}$, due to exact gradient calculation.

\begin{table*}[t]
\centering
\caption{Memory usage, runtime and final loss for coupled oscillator experiment. Mean $\pm$ standard deviation over three repeats.}
\vskip 0.1in
\begin{tabular}{l|c | c c c | c c c}
Method & \shortstack{Memory\\(checkpoints)} & \multicolumn{3}{c}{\shortstack{Runtime\\(min)}}  & \multicolumn{3}{c}{\shortstack{Loss\\($\times 10^{-3}$)}} \\ \hline
 & & Midpoint & Ralston3 & RK4 & Midpoint & Ralston3 & RK4 \\
\hline
Reversible & $\mathbf{2}$ & $\mathbf{14.3 \pm 3.1}$ & $\mathbf{13.2 \pm 0.3}$ & $\mathbf{19.7 \pm 6.6}$ & $\mathbf{1.0 \pm 0.2}$ &  $1.3 \pm 0.3$ &  $\mathbf{1.2 \pm 0.4}$ \\
\hline
\multirow[c]{5}{*}{Recursive} & $2$ & $632.2 \pm 20.0$ &  $645.2 \pm 7.2$ &  $654.9 \pm 0.8$ & \multirow[c]{5}{*}{$1.0 \pm 0.2$} & \multirow[c]{5}{*}{$\mathbf{1.2 \pm 0.1}$} & \multirow[c]{5}{*}{$1.4 \pm 0.3$} \\ 
 & $4$ & $99.0 \pm 10.7$ &  $95.7 \pm 1.9$ &  $97.5 \pm 0.1$ & & &\\ 
 & $8$ & $63.4 \pm 9.8$ &  $57.5 \pm 1.7$ &  $58.1 \pm 0.2$ & & &\\ 
 & $16$ & $53.8 \pm 8.8$ &  $48.7 \pm 2.8$ &  $48.6 \pm 0.1$ & & &\\ 
 & $31$ & $36.6 \pm 7.9$ &  $30.8 \pm 2.6$ &  $30.8 \pm 0.1$ & & &\\  
\end{tabular}
\label{tab:oscillator}
\end{table*} 

\subsection{Coupled oscillators}
The second experiment we consider investigates the performance of the reversible method for identifying dynamics from real-system data. We use the dataset from \cite{schmidt2009distilling}, designed to facilitate research in discovery of physical laws from data. Specifically, we first consider the real coupled oscillator system; the task is to learn a Neural ODE that approximates the dynamics.

The coupled oscillator data was sampled over $t\in[0, 3]$. The Neural ODE model was parameterised with a feed-forward neural network vector field and trained using both the reversible method and recursive checkpointing for a range of numerical solvers. The runtime, memory usage and final loss obtained is compared for both methods in Table \ref{tab:oscillator}.

We see that the reversible method is more that $2 \times$ faster on average than recursive checkpointing while using $15.5\times$ less memory. Further, for the same memory usage the reversible method is more than $40\times$ faster. The final loss obtained is also comparable across both methods.

\begin{table*}[b]
\centering
\caption{Memory usage, runtime, final loss and number of adaptive solver steps for double pendulum experiment. Mean $\pm$ standard deviation over five repeats.}
\vskip 0.1in
\begin{tabular}{l|c | c | c | c}
Method & \shortstack{Memory\\(checkpoints)} & \shortstack{Runtime\\(min)}  & \shortstack{Loss\\($\times 10^{-3}$)} & \shortstack{Solver steps\\(mean)} \\ \hline
Reversible & $\mathbf{2}$ &  $\mathbf{21.9 \pm 2.1}$ & $8.3 \pm 3.2$ & $445 \pm 14$ \\
\hline
\multirow[c]{5}{*}{Recursive} & $2$ &  $818.2 \pm 21.5$ & $9.5 \pm 2.0$ & $422 \pm 7$ \\ 
 & $4$ &  $135.2 \pm 7.1$ & $8.6 \pm 1.9$ & $417 \pm 7$ \\ 
 & $8$ &  $82.8 \pm 1.2$ & $12.8 \pm 7.4$ & $434 \pm 14$ \\ 
 & $16$ &  $70.8 \pm 3.7$ & $\mathbf{7.8 \pm 1.3}$ & $432 \pm 11$ \\ 
 & $32$ &  $62.4 \pm 2.7$ & $7.9 \pm 1.6$ & $429 \pm 10$ \\ 
\end{tabular}
\label{tab:pendulum}
\end{table*} 

\subsection{Chaotic double pendulum}
\label{subsec:pendulum}
Lastly, we consider identification of chaotic non-linear dynamics using the (real-system) chaotic double pendulum dataset from \cite{schmidt2009distilling}. We use this experiment to test the performance of adaptive reversible solvers. 

A feed-forward Neural ODE is trained to identify the dynamics of the double pendulum system. We use the adaptive solver Bogacki-Shampine $3/2$ (Bosh3) that contains an embedded second order method to provide local error estimates \cite{bogacki19893}.

The reversible Bosh3 method is compared to recursive checkpointing for a range of checkpoints $c$. For adaptive step sizes, the optimal number of checkpoints is unknown a-priori; we therefore choose to vary up to $c=32$. This guarantees that we are in the regime $c>\sqrt{n}$ to ensure $O(n\log n)$ runtime. The memory usage, runtime and best loss achieved is shown in Table \ref{tab:pendulum}.

We see from Table \ref{tab:pendulum} that reversible backpropagation is at least $2.9\times$ faster than recursive checkpointing while using $16\times$ less memory. For the same memory usage the reversible method is $37\times$ faster. The final loss is comparable across both methods.

Further, it is interesting to see that mean number of steps taken with the adaptive solver is comparable across both methods. One might have expected the reversible method to require a larger number of steps to achieve the same numerical error due to the decreased stability in comparison to the base solver. However, that is not observed in this experiment.

\section{Discussion}
\subsection{Implementation}
An implementation of the reversible solver method was written in JAX \cite{jax2018github}. The code can be found at \url{https://github.com/sammccallum/reversible-solvers}. This is part of ongoing work to implement reversible solvers in the popular differential equation solving library Diffrax \cite{kidger2022neural}.

\subsection{Performance improvements}
It was found that adding weight decay to the neural network vector field parameters improves numerical stability for both the reversible method and recursive checkpointing. This result has been previously observed in the neural ODE literature \cite{grathwohl2018ffjord}.

\subsection{Limitations}
The forward solve of the reversible method requires twice the computational cost of the base solver. However, this additional cost is overshadowed by the reduced cost of backpropagation, such that overall the reversible method is significantly faster than recursive checkpointing.

The reversible method can be less numerically stable than the recursive checkpointing algorithm. While the numerical stability of the reversible solvers presented is a strict improvement over previous reversible architectures, stability remains an area to be further improved.

\subsection{Future Work}
\subsubsection{Neural CDEs and SDEs}
The reversible method presented naturally extends to solving Neural CDEs and SDEs -- where the base solver step $\Psi$ is modified accordingly. For example, letting $\Psi$ correspond to the Euler-Heun method for solving Stratonovich SDEs then the reversible method in \eqref{eq:reversible-solver-forward} would correspond to a reversible Euler-Heun solver. These applications will be explored in future work.

\subsubsection{Partial Differential Equations}
The spatial finite difference approximation to PDEs results in a set of coupled ODEs. The number of ODEs is equal to the size of the discretization grid and the forward computation graph can therefore incur high memory cost \cite{morton2005numerical}. Application of reversible solvers to discretized PDEs would significantly decrease this memory cost and improve runtime.

\subsubsection{Implicit Models}
Neural ODEs have recently been situated in the larger model class of Implicit models. For example, Deep Equilibrium Models are an example of an implicit model where the solution is specified as the fixed/equilibrium point of some system \cite{bai2019deep}. 

Implicit models require a numerical scheme to determine the solution. If one can devise a numerical scheme that is algebraically reversible, then the improved runtime and memory-efficiency of reversible differential equation solvers may be extended to other implicit models.

\section*{Conclusion}
A class of reversible ODE solvers was introduced that significantly outperform recursive checkpointing algorithms for training Neural ODEs on both runtime and memory usage. The reversible method presented also strictly improves upon the stability properties and convergence order of previous reversible solvers. Further, the reversible method naturally handles adaptive step sizes and extends to the wider family of neural differential equations.


\bibliography{main}
\bibliographystyle{icml2025}

\newpage
\appendix
\onecolumn
\section{Proof of convergence (Theorem \ref{thm:convergence})}
\label{app:convergence}
We will consider the following ODE,
\begin{align}\label{eq:ode}
\frac{dy}{dt} & = f(y),\\[3pt]
y(0) & = y_0, \nonumber
\end{align}
where the state $(t, y)$ has been concatenated to $y\in \bR^n$ without loss of generality. The vector field is given by the Lipschitz function $f: \bR^n \rightarrow \bR^n$ and the initial condition is $y_0\in \bR^n$.

We start by defining the base solver $\Psi$ and associated error analysis. Throughout, the norm $\big\| \cdot \big\|$ corresponds to the Euclidean norm $\big\| \cdot \big\|_2$. \\

\begin{definition}
For $h_{\max} > 0$, consider the map $\Psi$ given by\label{def:order_solver}
\begin{align*}
\Psi : [-h_{\max}, h_{\max}]\times \bR^n & \rightarrow \bR^n,\\
(h, x) & \mapsto \Psi_h(x).
\end{align*}
We say that $\Psi$ is an ODE solver with order $\alpha > 0\,$ if there exists a constant $C_1 > 0$ such that
\begin{align}\label{eq:ode_solver_eq}
\big\|x(h) - \big(x + \Psi_h(x)\big)\big\| & \leq C_1 |h|^{\alpha + 1},
\end{align}
for $h\in [-h_{\max}, h_{\max}]\,$ where $x(h)$ denotes the solution $y(t)$ at time $t = |h|$ of the ODE:
\begin{align*}
y^\prime & = \mathrm{sgn}(h)\cdot f(y),\\[3pt]
y(0) & = x.
\end{align*}
\end{definition}

\begin{definition}
    \label{def:lip-solver}
    Let $\Psi$ denote an order $\alpha$ solver for an ODE governed by a Lipschitz vector field $f:\bR^n\rightarrow\bR^n$. We say that $\Psi$ is an ODE solver satisfying the Lipschitz condition if there exists a constant $C_2 > 0$ such that
    \begin{equation}
        \big\|\Psi_h(x)-\Psi_h(y)\big\| \leq C_2|h|\big\|x-y\big\|,
    \end{equation}
    for all $x, y \in \bR^n$ and $h\in [-h_{\max}, h_{\max}]$. In particular, this may be written as
    \begin{equation}
        \big\|\Psi_h\big\|_{\text{Lip-}1} \leq C_2|h|.
    \end{equation}
\end{definition}

We note that the Lipschitz condition defined in \ref{def:lip-solver} is true for all explicit Runge-Kutta methods; this result is shown in Section $6.1.4.2$ of \cite{stewart2022numerical}.

\begin{lemma}
Let $\Psi$ denote an order $\alpha$ ODE solver satisfying the Lipschitz condition. Then there exists a constant $C_3 > 0$ such that \label{lemma:forward_backward}
\begin{align*}
\big\|\Psi_h(x +  \Psi_{-h}(x)) +  \Psi_{-h}(x)\big\| &\leq C_3|h|^{\alpha+1},
\end{align*}
for $h\in[-h_{\max}, h_{\max}]$.
\end{lemma}
\begin{proof} By definition \ref{def:order_solver}, \ref{def:lip-solver} and the triangle inequality, we have
\begin{align*}
\big\|\Psi_h(x +  \Psi_{-h}(x)) +  \Psi_{-h}(x)\big\| & = \big\|\big(x(-h) + \Psi_h(x(-h)) - x\big) + \big(x +  \Psi_{-h}(x) - x(-h)\big) \\
&\mmm + \Psi_h(x +  \Psi_{-h}(x)) - \Psi_h(x(-h))\big\|\\[3pt]
&\leq \big\|x(-h) + \Psi_h(x(-h))  - x\big\| + \big\|x +  \Psi_{-h}(x) - x(-h)\big\|\\
&\mmm + \big\|\Psi_h(x +  \Psi_{-h}(x)) - \Psi_h(x(-h))\big\|\\[3pt]
&\leq 2C_1|h|^{\alpha+1} + C_2 |h| \big\|x +  \Psi_{-h}(x) - x(-h)\big\|\\[3pt]
&\leq \underbrace{C_1\big(2+ C_2\, h_{\max}\big)}_{=:\,C_3} |h|^{\alpha+1}.
\end{align*}
\end{proof}

Next, we bound the error between the reversible solver states in lemma \ref{lemma:yz_close}. \\

\begin{lemma}
For a fixed time horizon $T > 0$, we consider the ODE (\ref{eq:ode}) over the interval $[0,T]$. Suppose that $T = Nh$ were $N > 0$ denote the number of steps and $h > 0$ denotes the step size.
Consider the reversible numerical solution $\{(y_k, z_k)\}_{0\leq k\leq N}$ with $y_0 = z_0 = y(0)$ and, for $k\geq 0$,\label{lemma:yz_close}
\begin{align*}
y_{k+1} & := \lambda y_k + (1-\lambda)z_k + \Psi_{h}(z_k),\\[3pt]
z_{k+1} & := z_k - \Psi_{-h}(y_{k+1}),
\end{align*}
where $\lambda\in (0,1]$ and $\Psi$ denotes an order $\alpha$ ODE solver satisfying the Lipschitz condition. Then there exists $h_{\max}, C_4 > 0$ such that
\begin{align}\label{eq:yz_close}
\|y_k - z_k\| \leq C_4 h^{\alpha}\,,
\end{align}
for $h\in (0, h_{\max}]$ and $0\leq k\leq N$.
\end{lemma}
\begin{proof}
Letting $E_k^{yz} := \|y_k - z_k\|$, it follows from definition \ref{def:lip-solver} and lemma \ref{lemma:forward_backward} that
\begin{align*}
E_{k+1}^{yz}  & = \| y_{k+1} - z_{k+1}\| \\[3pt]
& = \big\|\lambda \big(y_k - z_k\big) + \Psi_h(z_k) + \Psi_{-h}(y_{k+1})\big\|\\[3pt]
& = \big\| \lambda \big(y_k - z_k\big) + \Psi_h(z_k) +  \Psi_{-h}(z_{k+1}) + \Psi_{-h}(y_{k+1})-  \Psi_{-h}(z_{k+1})\big\|\\[3pt]
& \leq \lambda \|y_k - z_k\| + \big\|\Psi_h(z_k) +  \Psi_{-h}(z_{k+1})\big\| + \big\|\Psi_{-h}(y_{k+1})-  \Psi_{-h}(z_{k+1})\big\|\\[3pt]
& \leq \lambda \|y_k - z_k\| + \big\|\Psi_h(z_{k+1} +  \Psi_{-h}(y_{k+1})) +  \Psi_{-h}(z_{k+1})\big\| + \|\Psi_{-h}\|_{\text{Lip-}1}\big\|y_{k+1}- z_{k+1}\big\|\\[3pt]
& \leq \lambda \|y_k - z_k\| +  \big\|\Psi_h(z_{k+1} +  \Psi_{-h}(z_{k+1})) +  \Psi_{-h}(z_{k+1})\big\| + C_2\, h\big\|y_{k+1}- z_{k+1}\big\| \\[1pt]
&\hspace{22.5mm} +  \big\|\Psi_h(z_{k+1} +  \Psi_{-h}(y_{k+1})) - \Psi_h(z_{k+1} +  \Psi_{-h}(z_{k+1}))\big\| \\[3pt]
& \leq \lambda E_k^{yz} + C_3 h^{\alpha + 1} + C_2 h E_{k+1}^{yz} +  C_2^2 h^2 E_{k+1}^{yz}\,.
\end{align*}
Therefore, we have the following inequality,
\begin{align*}
E_{k+1}^{yz}  \leq \frac{1}{1 - C_2 h - C_2^2 h^2}\big(\lambda E_k^{yz} + C_3 h^{\alpha + 1}\big),
\end{align*}
provided $C_2 h < \frac{1 +\sqrt{5}}{2}$. Note that $(1 - x - x^2)^{-1} \leq 16^x$ for $x\in[0, \frac{1}{2}]$. Hence, for $h \leq \frac{1}{2 C_2}$,
\begin{align*}
E_{k+1}^{yz}  \leq 16^{C_2 h}\big(\lambda E_k^{yz} + C_3 h^{\alpha + 1}\big).
\end{align*}
Applying the above inequality $k$ times and using $E_0^{yz} = 0$, we have
\begin{align*}
E_{k}^{yz} & \leq  16^{C_2 h} C_3\, h^{\alpha + 1}\frac{16^{C_2 k h} - 1}{16^{C_2 h} - 1}\\
& = C_3 \frac{16^{C_2 h}}{\log(16)C_2}\big(16^{C_2 kh} - 1\big)\frac{\log(16)C_2h}{e^{\log(16)C_2 h} - 1} h^{\alpha}\\
& \leq \underbrace{C_3 \frac{16^{C_2 h}}{\log(16)C_2}\big(16^{C_2 T} - 1\big)}_{=:\,C_4} h^\alpha,
\end{align*}
where we used $kh \leq T$ and the inequality $1 + x \leq e^x$.
\end{proof}

Bringing the above analysis together, we now prove Theorem \ref{thm:convergence}.

\begin{theorem*}
For a fixed time horizon $T > 0$, we consider the ODE (\ref{eq:ode}) over the interval $[0,T]$. Let $T = Nh$ where $N > 0$ denotes the number of steps and $h > 0$ is the step size.
Consider the reversible numerical solution $\{(y_k, z_k)\}_{0\leq k\leq N}$ with $y_0 = z_0 = y(0)$ and, for $k\geq 0$,
\begin{align}
y_{k+1} & := \lambda y_k + (1-\lambda)z_k + \Psi_{h}(z_k),\label{eq:reversible_ode3}\\[3pt]
z_{k+1} & := z_k - \Psi_{-h}(y_{k+1}),\label{eq:reversible_ode4}
\end{align}
where $\lambda\in (0,1]$ and $\Psi$ denotes an order $\alpha$ ODE solver satisfying the Lipschitz condition. Then there exist constants $h_{\max}, C_5 > 0$ such that, for $h\in (0, h_{\max}]$,
\begin{align}\label{eq:yy_close}
\|y_k - y(t_k)\| \leq C_5 h^{\alpha}\,,
\end{align}
for $0\leq k\leq N$ where $t_k := kh$.
\end{theorem*}
\begin{proof}

Let $x_k := \lambda^{N-k} y_k + (1-\lambda^{N-k})z_k$ and consider the error $E_k := \| x_k - y(t_k)\|$. Then, by the triangle inequality and previous lemmas, we have
\begin{align*}
E_{k+1} & = \big\| \lambda^{N-(k+1)} y_{k+1} + (1-\lambda^{N-(k+1)})z_{n+1} - y(t_{k+1})\big\|\\[3pt]
& = \big\| \lambda^{N-(k+1)} \big(\lambda y_k + (1-\lambda)z_k + \Psi_{h}(z_k)\big) + (1-\lambda^{N-(k+1)})\big(z_k - \Psi_{-h}(y_{k+1})\big) - y(t_{k+1})\big\|\\[3pt]
& = \big\| \lambda^{N-k} y_k + (1-\lambda^{N-k})z_k + \lambda^{N-(k+1)}\Psi_{h}(z_k) - (1-\lambda^{N-(k+1)})\big(\Psi_{-h}(y_{k+1})\big) - y(t_{k+1})\big\|\\[3pt]
& \leq E_k + \big\| \,y(t_k) + \lambda^{N-(k+1)}\Psi_{h}(z_k) - (1-\lambda^{N-(k+1)})\big(\Psi_{-h}(y_{k+1})\big) - y(t_{k+1})\big\|\\[3pt]
& \leq E_k + \lambda^{N-(k+1)}\big\|\Psi_{h}(z_k) + \Psi_{-h}(y_{k+1})\big\| + \big\|y(t_k) - \big(y(t_{k+1}) + \Psi_{-h}(y_{k+1})\big)\big\|\\[3pt]
&  \leq E_k + \lambda^{N-(k+1)}\big\|\Psi_{h}(z_{k+1} + \Psi_{-h}(y_{k+1})) + \Psi_{-h}(y_{k+1})\big\|\\[2pt]
&\hspace{9mm} + \big\|y(t_k) - \big(y(t_{k+1}) + \Psi_{-h}(y(t_{k+1})\big)\big\| + \big\|\Psi_{-h}(y(t_{k+1})) - \Psi_{-h}(y_{k+1})\big\|\\
& \leq E_k + \lambda^{N-(k+1)}\big\|\Psi_{h}(y_{k+1} + \Psi_{-h}(y_{k+1})) + \Psi_{-h}(y_{k+1})\big\|\\[2pt]
&\hspace{9mm} + \lambda^{N-(k+1)}\big\|\Psi_{h}(z_{k+1} + \Psi_{-h}(y_{k+1})) - \Psi_{h}(y_{k+1} + \Psi_{-h}(y_{k+1}))\big\|\\[2pt]
&\hspace{9mm} + C_1 h^{\alpha + 1} + \big\|\Psi_{-h}(y(t_{k+1})) - \Psi_{-h}(x_{k+1})\big\| + \big\|\Psi_{-h}(x_{k+1}) - \Psi_{-h}(y_{k+1})\big\|\\[3pt]
& \leq E_k + \lambda^{N-(k+1)} C_3 h^{\alpha+1} + \lambda^{N-(k+1)} C_2 h\|z_{k+1} - y_{k+1}\|\\[2pt]
& \hspace{9mm} + C_1 h^{\alpha + 1} + C_2 h\, E_{k+1} + C_2 h\|x_{k+1} - y_{k+1}\|\\[3pt]
& \leq E_k + \lambda^{N-(k+1)} C_3 h^{\alpha+1} + \lambda^{N-(k+1)} C_2 h\|y_{k+1} - z_{n+1}\|\\[2pt]
& \hspace{9mm} + C_1 h^{\alpha + 1} + C_2 h\, E_{k+1} + (1 - \lambda^{N-(k+1)})C_2 h\|y_{k+1} - z_{n+1}\|\\[3pt]
& \leq E_k +  C_2 h\, E_{k+1} + (C_1 +  C_2 C_4 + C_3)h^{\alpha+1}.
\end{align*}
Therefore, provided $C_2 h < 1$,
\begin{align*}
E_{k+1} \leq \frac{1}{1-C_2 h} E_k + \frac{C_1 +  C_2 C_4 + C_3 }{1-C_2 h}h^{\alpha+1}.
\end{align*}
Since $(1-x)^{-1} \leq 1 + 2x \leq e^{2x}$ for $x\in[0,\frac{1}{2}]$, we have for $h \leq \frac{1}{2C_2}$,
\begin{align*}
E_{k+1} \leq e^{2C_2 h} E_k + (C_1 +  C_2 C_4 + C_3 )(1 + 2C_2 h)h^{\alpha+1}.
\end{align*}
Applying the above inequality $k$ times and using $E_0 = 0$, we have
\begin{align*}
E_k & \leq (C_1 +  C_2 C_4 + C_3 )(1 + 2C_2 h)h^{\alpha+1}\frac{e^{2C_2 kh} - 1}{e^{2C_2 h} - 1}\\[3pt]
& = (C_1 +  C_2 C_4 + C_3 )(1 + 2C_2 h)\frac{e^{2C_2 kh} - 1}{2C_2}\frac{2C_2 h}{e^{2C_2 h} - 1} h^{\alpha}\\[3pt]
& \leq (C_1 +  C_2 C_4 + C_3 )(1 + 2C_2 h_{\max})\frac{e^{2C_2 T} - 1}{2C_2} h^{\alpha},
\end{align*}
where we used $kh \leq T$ and the inequality $1 + x \leq e^x$. Therefore
\begin{align*}
\|y_k - y(t_k)\| & \leq \|y_k - x_k\| + \|x_k - y(t_k)\|\\[3pt]
& = E_k + (1 - \lambda^{N-k})\| y_k - z_k\|\\[3pt]
& \leq \underbrace{\left((C_1 +  C_2 C_4 + C_3 )(1 + 2C_2 h_{\max})\frac{e^{2C_2 T} - 1}{2C_2} + C_4\right)}_{=:\,C_5} h^\alpha.
\end{align*}
\end{proof}

\section{Proof of stability (Theorem \ref{thm:stability})}
\label{app:stability}
For convenience we provide the description of Theorem \ref{thm:stability} here.
\begin{theorem*}
    Let $\Psi$ be given by an explicit Runge-Kutta solver. Then the reversible numerical solution $\{y_n, z_n\}_{n\geq 0}$ given by
    \begin{equation*}
        \begin{aligned}
            y_{n+1} &= \lambda y_n+(1-\lambda)z_n + \Psi_h(t_n, z_n), \\
            z_{n+1} &= z_n - \Psi_{-h}(t_{n+1}, y_{n+1}),
        \end{aligned}
    \end{equation*}
    is linearly stable iff
    \begin{equation*}
        |\Gamma| < 1+\lambda,
    \end{equation*}
    where
    \begin{equation*}
        \Gamma = 1+\lambda - (1-\lambda)R(-h\alpha)-R(-h\alpha)R(h\alpha).
    \end{equation*}
\end{theorem*}
\begin{proof}
    Firstly, we re-write the reversible method applied to the linear test problem (definition \ref{def:linear-stability}) using the Runge-Kutta transfer function. We get
    \begin{equation*}
        \begin{aligned}
            y_{n+1}&=\lambda y_n+(1-\lambda)z_n+R(h\alpha)z_n, \\
            z_{n+1}&=z_n-R(-h\alpha)y_{n+1}.
        \end{aligned}
    \end{equation*}
    Writing the coupled system as a matrix-vector product results in
    \begin{equation*}
        \begin{pmatrix} y_{n+1} \\ z_{n+1} \end{pmatrix} = \begin{pmatrix} \lambda & 1-\lambda + R(h\alpha) \\ -\lambda R(-h \alpha) & 1 + (1-\lambda)R(h\alpha) - R(-h\alpha)R(h\alpha) \end{pmatrix}\begin{pmatrix} y_n \\ z_n \end{pmatrix}.
    \end{equation*}
    For stability we require that the eigenvalues of the coupling matrix are by magnitude less than one. The eigenvalues are given by the characteristic equation,
    \begin{equation*}
        P(e_k) = e_k^2-\left(1+\lambda-(1-\lambda)R(-h\alpha)-R(-h\alpha)R(h\alpha)\right)e_k+\lambda=0.
    \end{equation*}
    
    \textbf{Routh-Hurwitz test}
    
    The Routh-Hurwitz test is a condition on the coefficients of a polynomial such that the roots lie in the left half complex plane. Instead we require that the roots lie in the unit disc. Therefore, we apply the following map to the characteristic equation $P(e_k)$,
    \begin{equation*}
        Q(e_k)=(e_k-1)^2P\left(\frac{e_k+1}{e_k-1}\right),
    \end{equation*}
    such that if $Q$ is Routh-Hurwitz stable, then the roots of $P$ lie in the open unit disc (referred to as Schur stable) \cite{bhattacharyya1995robust}.
    
    The polynomial $Q$ obtained by the transformation is given by
    \begin{equation*}
        Q(e_k)=\left(1+\lambda-\Gamma\right)e_k^2+2(1-\lambda)e_k+(1+\lambda+\Gamma)=0,
    \end{equation*}
    where
    \begin{equation*}
        \Gamma = 1+\lambda-(1-\lambda)R(-h\alpha)-R(-h\alpha)R(h\alpha).
    \end{equation*}
    
    We proceed to check when $Q$ is Routh-Hurwitz stable. For quadratic polynomials, the Routh-Hurwitz test reduces to checking that the coefficients of $Q$ are all positive. Therefore, for stability, we require
    \begin{equation*}
        \begin{aligned}
            1+\lambda-\Gamma &> 0 \hspace{8mm} \text{(i)} \\
            2(1-\lambda) &> 0 \hspace{8mm} \text{(ii)} \\
            1+\lambda+\Gamma &> 0 \hspace{8mm} \text{(iii)}
        \end{aligned}
    \end{equation*}
    Condition (ii) is immediately satisfied by the definition of $\lambda \in (0, 1)$. Condition (i) and (iii) may be combined into the condition
    \begin{equation*}
        |\Gamma| < 1 + \lambda.
    \end{equation*}
\end{proof}

\section{Experimental details}
All experiments were run on one Dual NVIDIA RTX A6000, 48 GB.

For all experiments the Neural ODE model is defined with a small feed-forward neural network vector field with two layers and hidden size $10$. The activation function used was $\tanh$. Each training run was performed for multiple repeats with random network initializations.

The reversible solver was initialized with coupling parameter $\lambda=0.99$.

\subsection{Chandrasekhar’s White Dwarf Equation}
Training data was generated by simulating the white dwarf system over $r\in [0, 5]$ with $C=0.001$ for $1000$ time steps. The Neural ODE model was trained to minimise mean squared error over $1000$ time steps. The number of training steps was $1000$ and the optimizer was AdamW with learning rate $10^{-2}$ and weight decay $10^{-5}$ \cite{loshchilov2019decoupled}.

\subsection{Coupled oscillators}
The experimental coupled oscillator data from \cite{schmidt2009distilling} was linearly interpolated, sampled at $500$ time steps over $t\in[0, 3]$ and normalized. The Neural ODE model was trained to minimise mean squared error over $500$ time steps. The number of training steps was $10000$ and the optimizer was AdamW with learning rate $10^{-2}$ and weight decay $10^{-5}$.

\subsection{Chaotic double pendulum}
The experimental double pendulum data from \cite{schmidt2009distilling} was linearly interpolated, sampled at $500$ time steps over $t\in[0, 2]$ and normalized. The Neural ODE model was trained to minimise mean squared error over an adaptively selected number of time steps. The number of training steps was $10000$ and the optimizer was AdamW with learning rate $10^{-2}$ and weight decay $10^{-5}$.

The adaptive time stepping algorithm was a PID controller from \cite{soderlind2003digital} with error estimates provided by the embedded second order method from Bogacki-Shampine \cite{bogacki19893}. The PID controller absolute and relative error tolerances were $10^{-6}$.


\end{document}